\documentclass[11pt,a4paper]{article}
\usepackage[hyperref]{acl2020}
\usepackage{amsmath, amssymb, amsthm, amscd, amsfonts}
\usepackage{times}
\usepackage{latexsym}
\usepackage{graphicx}
\usepackage{epstopdf}
\usepackage{algorithm}
\usepackage{algorithmicx}
\usepackage{algpseudocode}
\usepackage{url}
\usepackage{subfigure}
\usepackage{CJK}
\usepackage{url}
\usepackage{enumerate}
\usepackage{multirow}
\usepackage{bm}
\usepackage{array}
\usepackage[OT1]{fontenc}

\usepackage{url}

\usepackage{microtype}

\aclfinalcopy 
\def\aclpaperid{388} 

\newcommand{\toprightarrow}[1]{\mathord{\buildrel{\lower3pt\hbox{$\scriptscriptstyle\rightarrow$}}\over#1} }
\newcommand{\topleftarrow}[1]{\mathord{\buildrel{\lower3pt\hbox{$\scriptscriptstyle\leftarrow$}}\over#1} }

\newtheorem{myDef}{Definition}

\newtheorem{myProp}{Proposition}

\title{A Span-based Linearization for Constituent Trees}

\author{Yang Wei, Yuanbin Wu, and Man Lan \\
    School of Computer Science and Technology \\
    East China Normal University \\
    \texttt{godweiyang@gmail.com}\quad \texttt{\{ybwu,mlan\}@cs.ecnu.edu.cn}\\}

\date{}

\begin{document}
\begin{CJK*}{GBK}{song}
    \maketitle
    \begin{abstract}
        We propose a novel linearization of a constituent tree,
        together with a new locally normalized model.
        For each split point in a sentence,
        our model computes the normalizer on all spans ending with that split point,
        and then predicts a tree span from them.
        Compared with global models,
        our model is fast and parallelizable.
        Different from previous local models,
        our linearization method is tied on the spans directly and
        considers more local features when performing span prediction,
        which is more interpretable and effective.
        Experiments on PTB (95.8 F1) and CTB (92.4 F1) show that our model
        significantly outperforms existing local models
        and efficiently achieves competitive results with global models.
    \end{abstract}

    \section{Introduction}
    \label{Sec:Introduction}
    Constituent parsers map natural language sentences to hierarchically organized spans
    \citep{DBLP:conf/emnlp/CrossH16}.
    According to the complexity of decoders, two types of parsers have been studied,
    globally normalized models which normalize probability of a constituent tree on
    the whole candidate tree space (e.g. chart parser \citep{DBLP:conf/acl/SternAK17})
    and locally normalized models which normalize tree probability on smaller subtrees or spans.
    It is believed that global models have better parsing performance \citep{DBLP:conf/naacl/GaddySK18}.
    But with the fast development of neural-network-based feature representations
    \citep{DBLP:journals/neco/HochreiterS97, DBLP:conf/nips/VaswaniSPUJGKP17},
    local models are able to get competitive parsing accuracy while enjoying fast training
    and testing speed, and thus become an active research topic in constituent parsing.

    Locally normalized parsers usually rely on tree decompositions or linearizations.
    From the \textbf{perspective of decomposition}, the probability of trees can be factorized,
    for example, on individual spans.
    \citet{DBLP:conf/coling/TengZ18} investigates such a model
    which predicts probability on each candidate span.
    It achieves quite promising parsing results, while the simple local probability
    factorization still leaves room for improvements.
    From the \textbf{perspective of linearization}, there are many ways to transform a structured
    tree into a shallow sequence.
    As a recent example, \citet{DBLP:conf/acl/BengioSCJLS18} linearizes a tree with a
    sequence of numbers, each of which indicates words' syntactic distance in the tree
    (i.e., height of the lowest common ancestor of two adjacent words).
    Similar ideas are also applied in
    \citet{DBLP:conf/nips/VinyalsKKPSH15}, \citet{DBLP:conf/emnlp/ChoeC16}
    and transition-based systems \citep{DBLP:conf/emnlp/CrossH16, DBLP:journals/tacl/LiuZ17a}.
    With tree linearizations,
    the training time can be further accelerated to $\mathcal{O}(n)$,
    but the parsers often sacrifice a clear connection with original spans in trees,
    which makes both features and supervision signals from spans hard to use.

    \begin{figure*}[t]
        \centering
        \subfigure[Original parsing tree.]{
            \begin{minipage}[t]{0.272\textwidth}
                \includegraphics[width=\textwidth]{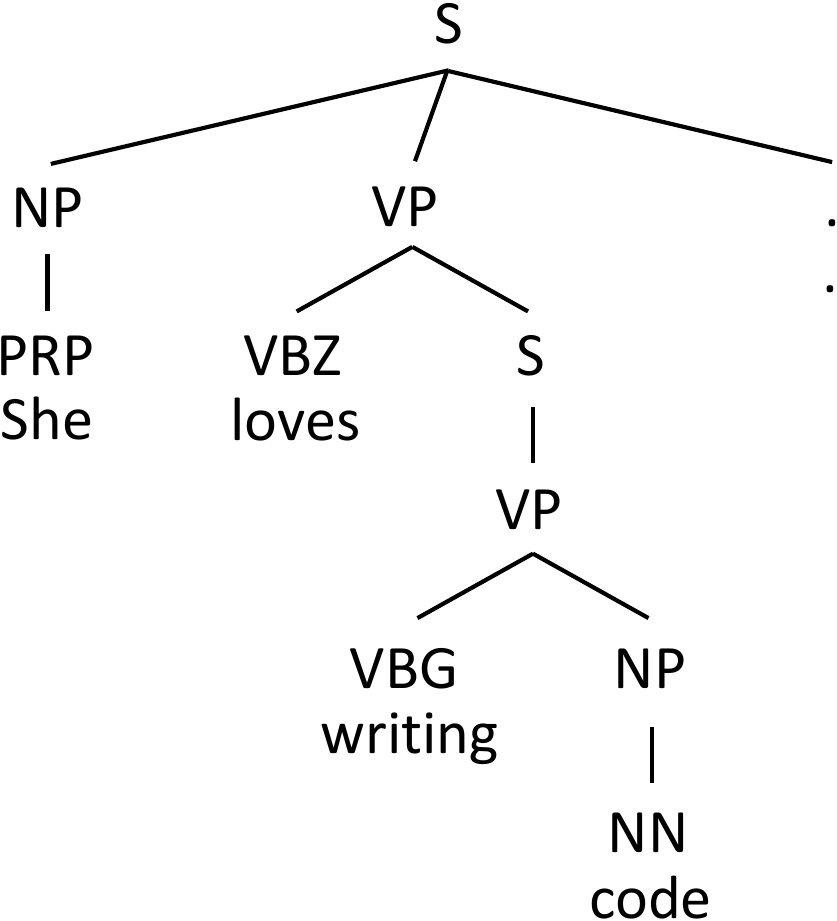}
            \end{minipage}
            \label{Fig:example_1}
        }
        \subfigure[Right binary tree.]{
            \begin{minipage}[t]{0.287\textwidth}
                \includegraphics[width=\textwidth]{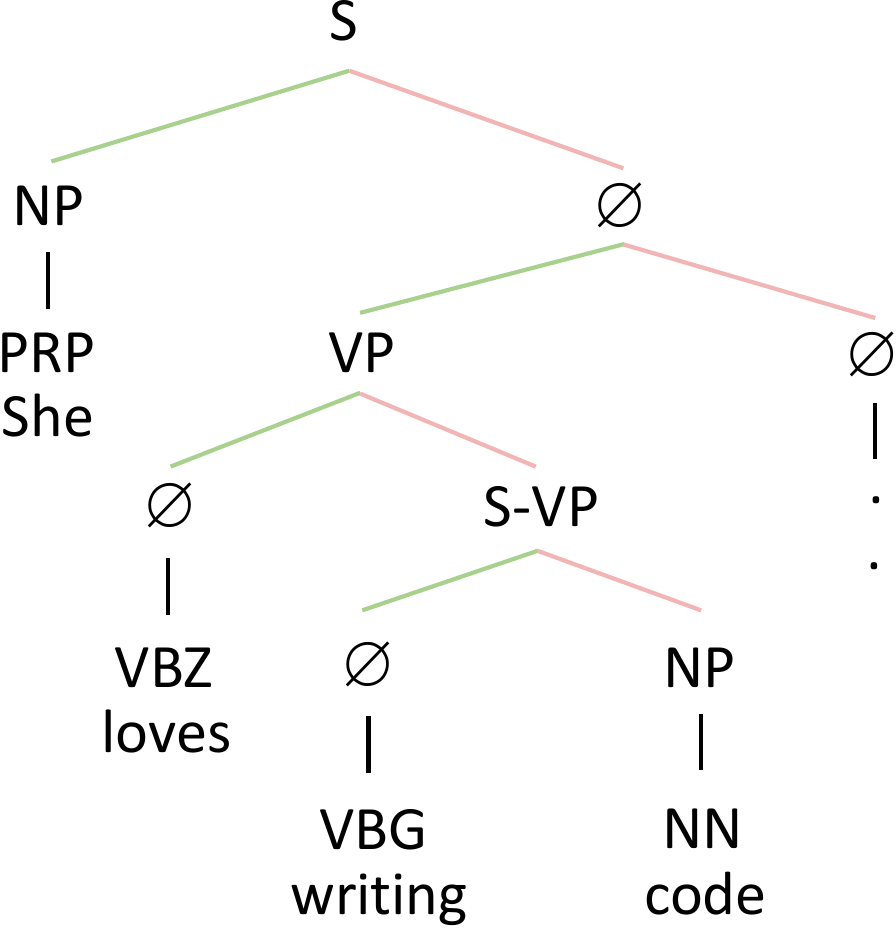}
            \end{minipage}
            \label{Fig:example_2}
        }
        \subfigure[Span table and linearization.]{
            \begin{minipage}[t]{0.387\textwidth}
                \includegraphics[width=\textwidth]{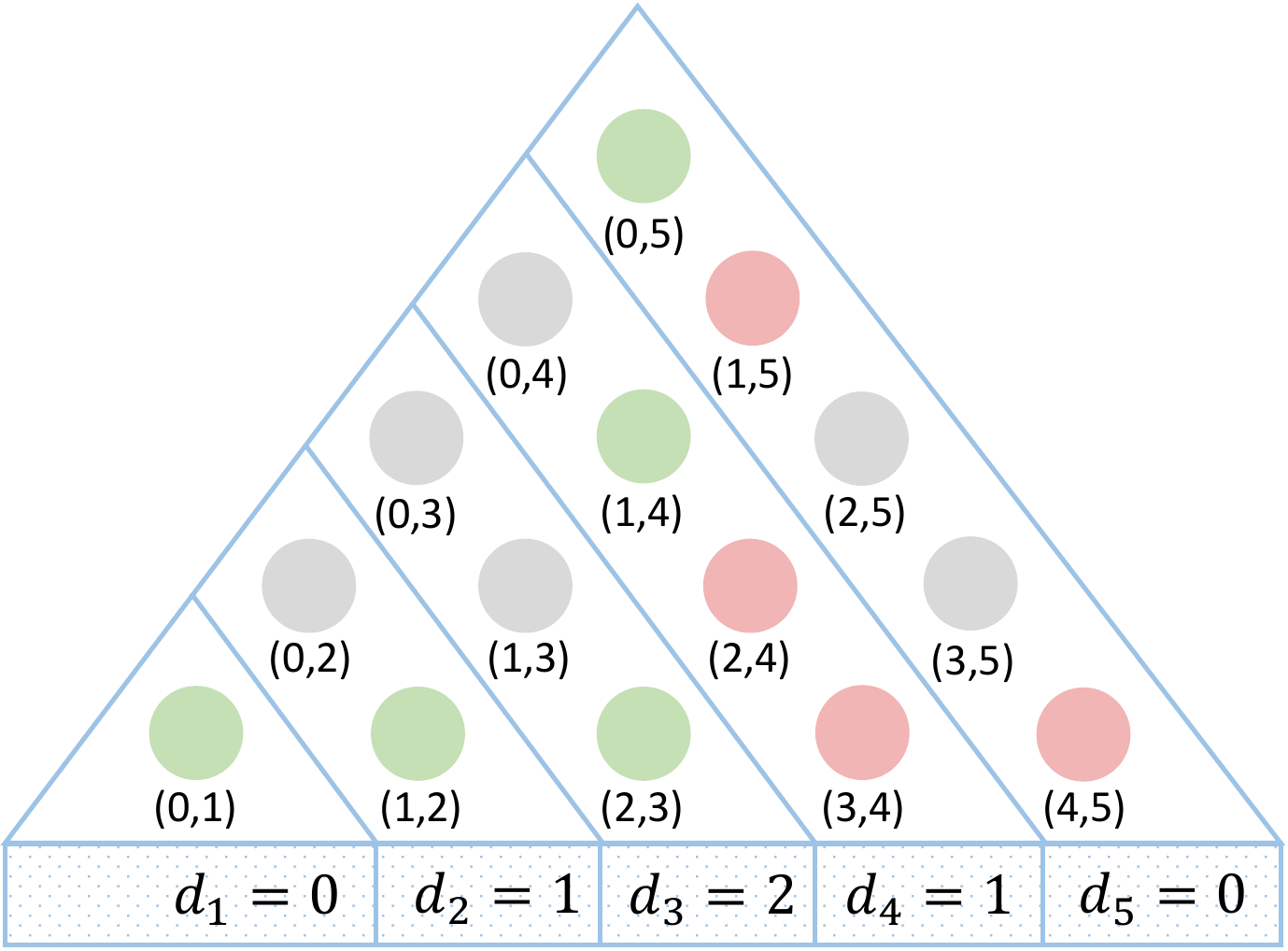}
            \end{minipage}
            \label{Fig:example_3}
        }
        \caption{The process of generating the linearization
            of the sentence ``\emph{She loves writing code .}''.
            Given an original parsing tree (a), we firstly convert it to a right binary tree
            by recursively combining the rightmost two children (b).
            Then, we represent the tree as a span table, and divide it into five parts according to
            the right boundaries of the spans (c).
            Green and red circles represent left and right child spans respectively.
            Gray circles represent spans which do not appear in the tree.
            In each part, there is only one longest span (green circles),
            thus the corresponding value of that part is just the left boundary of the green circle.}
        \label{Fig:example}
    \end{figure*}

    In this work, we propose a novel linearization of constituent trees
    tied on their span representations.
    Given a sentence $\mathcal{W}$ and its parsing tree $\mathcal{T}$,
    for each split point after $w_i$ in the sentence, we assign it a parsing target $d_i$,
    where $(d_i, i)$ is the longest span ending with $i$ in $\mathcal{T}$.
    We can show that, for a binary parsing tree, the set
    $\{(d_i, i)\}$ includes all left child spans in $\mathcal{T}$.
    Thus the linearization is actually sufficient to recover a parsing
    tree of the sentence.

    Compared with prior work, the linearization is directly based on tree spans,
    which might make estimating model parameters easier.
    We also build a different local normalization compared with
    the simple per-span-normalization in \citet{DBLP:conf/coling/TengZ18}.
    Specifically, the probability $P(d_i | i)$ is normalized on
    all candidate split points on the left of $i$.
    The more powerful local model can help to further improve parsing performance
    while retaining the fast learning and inference speed
    (with a greedy heuristic for handling illegal sequences,
    we can achieve $\mathcal{O}(n \log n)$ average inference complexity).
    We perform experiments on PTB and CTB.
    The proposed parser significantly outperforms existing locally normalized models,
    and achieves competitive results with state-of-the-art global models
    (95.8 F1 on PTB and 92.4 F1 on CTB).
    We also evaluate how the new linearization helps parse spans with
    different lengths and types.

    To summarize, our main contributions include:
    \begin{itemize}
        \item Proposing a new linearization which has clear interpretation
              (Section \ref{Sec:Preliminaries}).
        \item Building a new locally normalized model with constraints on span scores
              (Section \ref{Sec:OurModel}).
        \item Compared with previous local models, the proposed parser achieves
              better performance (competitive with global models) and has faster parsing speed
              (Section \ref{Sec:Experiments}).
    \end{itemize}

    \section{Tree Linearization}
    \label{Sec:Preliminaries}
    We first prepare some notations.
    Let $\mathcal{W} = (w_1, w_2, \ldots, w_n)$ be a sentence,
    $\mathcal{T}$ be its binary constituent tree and
    $A_{ij} \to B_{ik} C_{kj}$ be a derivation in $\mathcal{T}$.
    Denote $(i, j) (0 \le i < j \le n)$ to be a span from $w_{i+1}$ to $w_{j}$
    (for simplicity, we ignore the label of a span).

    \begin{myDef}
        Given a sentence $\mathcal{W}$ and its tree $\mathcal{T}$,
        we call $\mathcal{D} = (d_1, d_2, \ldots, d_n)$ a linearization of $\mathcal{T}$,
        where $d_i \in \{0, 1, \ldots, i-1\}$
        and $(d_i, i)$ is the longest span ending with $i$ in $\mathcal{T}$.
    \end{myDef}

    Clearly, there is only one such linearization for a tree.
    We have an equal definition of $\mathcal{D}$,
    which shows the span $(d_i, i)$ is a left child span.

    \begin{myProp}
        \label{Prop:left_span}
        Given a tree $\mathcal{T}$, the set of spans
        $\{(d_i, i) \mid i = 1, 2, \ldots, n\}$ is equal to the set of left child spans
        \footnote{The root node is also regarded as a left child span.}
        \[
            \mathcal{S} = \{(i, j) \mid \exists A_{ik} \to B_{ij} C_{jk}\} \cup \{(0, n)\}.
        \]
    \end{myProp}

    \begin{proof}
        First, for each $j$, there is only one left child span $(i, j)$ ending with $j$,
        otherwise if $(i', j)$ is a left child span with $i' \neq i$ (e.g. $i' < i$),
        $(i, j)$ must also be a right child span.
        Therefore $|\mathcal{S}| = n$.
        Similarly, if $i \neq d_j$, $(i, j)$ should be a right child span of $(d_j, j)$.
    \end{proof}

    Thus we can generate the linearization using Algorithm \ref{Alg:T2L}.
    For span $(i, j)$ and its gold split $k$, we can get $d_k = i$.
    Then we recursively calculate the linearization of span $(i, k)$ and $(k, j)$.
    Note that the returned linearization $\mathcal{D}$ does not contain $d_n$,
    so we append zero ($d_n = 0$ for the root node) to the end as the final linearization.
    Figure \ref{Fig:example} is a generation process of sentence
    ``\emph{She loves writing code .}''.
    From the span table, it is obvious that
    there is only one left child span (green circles)
    ending with the same right boundary.

    In the following discussions, we will use $\mathcal{D}$ and $\mathcal{S}$ interchangeably.
    Next, we show two properties of a legal $\mathcal{D}$.

    \begin{myProp}
        \label{Prop:legal}
        A linearization $\mathcal{D}$ can recover a tree $\mathcal{T}$ iff.
        \begin{enumerate}
            \item {$0 \le d_i < i, \forall 1 \le i \le n$.}
            \item {$d_j$ is not in the range $(d_i, i)$, $\forall j > i$.}
        \end{enumerate}
    \end{myProp}

    \begin{algorithm}[t]
        \caption{ Tree linearization. }\label{Alg:T2L}
        \begin{algorithmic}[1]
            \Function {Linearization}{$i, j, \mathcal{T}$}
            \If {$i + 1 = j$}
            \State $\mathcal{D} \gets []$
            \Else
            \State $k \gets $ the split point of span $(i, j)$ in $\mathcal{T}$
            \State $\mathcal{D}^l \gets $ \Call{Linearization}{$i, k, \mathcal{T}$}
            \State $\mathcal{D}^r \gets $ \Call{Linearization}{$k, j, \mathcal{T}$}
            \State $\mathcal{D} \gets \mathcal{D}^l \oplus [i] \oplus \mathcal{D}^r$
            \EndIf
            \State \Return $\mathcal{D}$
            \EndFunction
        \end{algorithmic}
    \end{algorithm}

    \begin{proof}
        The necessity is obvious.
        We show the sufficiency by induction on the sentence length.
        When $n = 1$, the conclusion stands.
        Assuming for all linearizations with length less than $n$,
        property 1 and 2 lead to a well-formed tree,
        and now consider a linearization with length $n$.

        Define $k = \max\{k' \mid d_{k'} = 0, k' < n\}$.
        Since $d_1 = 0$ (by property 1), $k$ is not none.
        We split the sentence into $(0, k)$, $(k, n)$,
        and claim that after removing $(0, n)$,
        the spans in $\mathcal{D}$ are either in $(0, k)$ or $(k, n)$,
        thus by induction we obtain the conclusion.
        To validate the claim, for $k' < k$, by property 1, we have $d_{k'} < k' < k$,
        thus $(d_{k'}, k')$ is in $(0, k)$.
        For $k' > k$, by property 2, either $d_{k'} \ge k$ or $d_{k'} = 0$.
        Since $k$ is the largest index with $d_k = 0$, we have $d_{k'} \neq 0$,
        which means $(d_{k'}, k')$ is in $(k, n)$.
        Therefore, we show the existence of a tree from $\mathcal{D}$.
        The tree is also unique, because if two trees $\mathcal{T}$ and $\mathcal{T'}$
        have the same linearization, by Proposition \ref{Prop:left_span},
        we have $\mathcal{T} = \mathcal{T'}$.
    \end{proof}

    Proposition \ref{Prop:legal} also suggests a top-down algorithm (Algorithm \ref{Alg:L2T})
    for performing tree inference given a \emph{legal linearization}.
    For span $(i, j)$ (with label $\ell(i, j)$),
    we find the rightmost split $k$ satisfying $d_k = i$,
    and then recursively decode the two sub-trees rooted at span $(i, k)$ and $(k, j)$, respectively.
    When $\mathcal{D}$ does not satisfy property 2
    (our model can ensure property 1),
    one solution is to seek a minimum change of $\mathcal{D}$ to make it legal.
    However, it is reduced to a minimum vertex cover problem
    (regarding each span $(d_i, i)$ as a point,
    if two spans violate property 2, we connect an edge between them.
    ).
    We can also slightly modify Algorithm \ref{Alg:L2T} to perform an approximate inference
    (Section \ref{Sec:TreeInference}).


    Finally we need to deal with the linearization of non-binary trees.
    For spans having more than two child spans,
    there is no definition for their middle child spans
    whether they are left children or right children,
    thus Proposition \ref{Prop:left_span} might not stand.
    We recursively combine two adjacent spans from right to left using an empty label $\varnothing$.
    Then the tree can be converted to a binary tree \citep{DBLP:conf/acl/SternAK17}.
    For a unary branch, we treat it as a unique span with a new label
    which concatenates all the labels in the branch.

    \begin{algorithm}[t]
        \caption{ Tree reconstruction. }\label{Alg:L2T}
        \begin{algorithmic}[1]
            \Function {Tree}{$i, j, \mathcal{D}$}
            \If {$i + 1 = j$}
            \State $\text{node} \gets \text{Leaf}(w_j, \ell(i, j))$
            \Else
            \State $k \gets \max{\{k' \mid d_{k'} = i, i < k' < j\}}$
            \State $\text{child}_l \gets $ \Call{Tree}{$i, k, \mathcal{D}$}
            \State $\text{child}_r \gets $ \Call{Tree}{$k, j, \mathcal{D}$}
            \State $\text{node} \gets \text{Node}(\text{child}_l, \text{child}_r, \ell(i, j))$
            \EndIf
            \State \Return $\text{node}$
            \EndFunction
        \end{algorithmic}
    \end{algorithm}

    \section{The Parser}
    \label{Sec:OurModel}
    In this section, we introduce our encoder, decoder and inference algorithms
    in detail.
    Then we compare our normalization method with two other methods,
    globally normalized and existing locally normalized methods.

    \subsection{Encoder}
    \label{Sec:Encoder}
    We represent each word $w_i$ using three pieces of information,
    a randomly initialized word embedding $\bm{e}_i$,
    a character-based embedding $\bm{c}_i$ obtained by a character-level LSTM
    and a randomly initialized part-of-speech tag embedding $\bm{p}_i$.
    We concatenate these three embeddings to generate a representation of word $w_i$,
    \[
        \bm{x}_i = [\bm{e}_i; \bm{c}_i; \bm{p}_i].
    \]

    To get the representation of the split points, the word representation matrix
    ${\rm\bf{X}} = [\bm{x}_1, \bm{x}_2, \ldots, \bm{x}_n]$ is fed into a
    bidirectional LSTM or Transformer \citep{DBLP:conf/nips/VaswaniSPUJGKP17} firstly.
    Then we calculate the representation of the split point between $w_i$ and $w_{i+1}$
    using the outputs from the encoders,
    \begin{equation}
        \label{Eq:encoder_output}
        \bm{h}_i = [\bm{\toprightarrow{h}}_i; \bm{\topleftarrow{h}}_{i+1}].
    \end{equation}
    Note that for Transformer encoder, $\bm{\toprightarrow{h}}_i$
    is calculated in the same way as \citet{DBLP:conf/acl/KleinK18}.

    \begin{figure*}[t]
        \centering
        \subfigure[Global normalization.]{
            \begin{minipage}[t]{0.317\textwidth}
                \includegraphics[width=\textwidth]{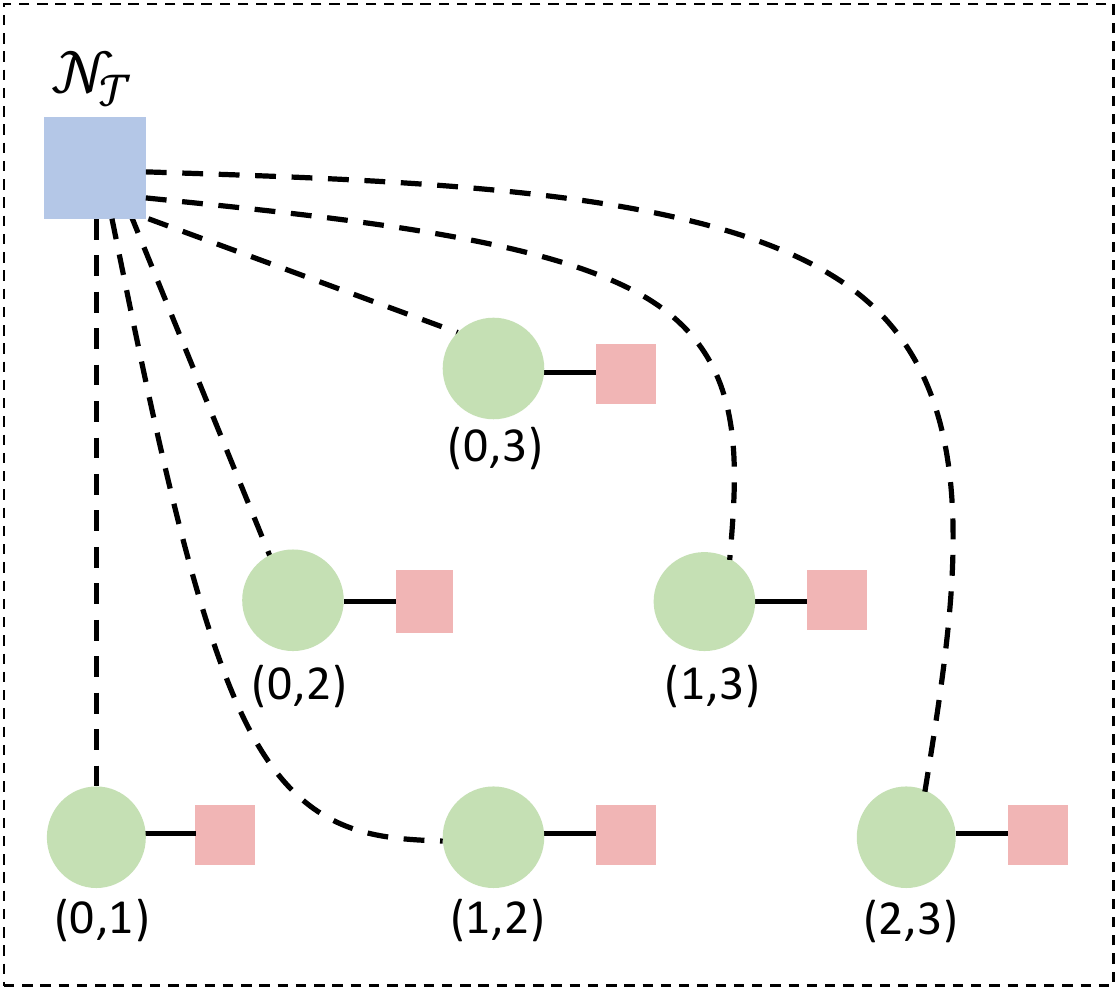}
            \end{minipage}
            \label{Fig:model_1}
        }
        \subfigure[Local normalization.]{
            \begin{minipage}[t]{0.317\textwidth}
                \includegraphics[width=\textwidth]{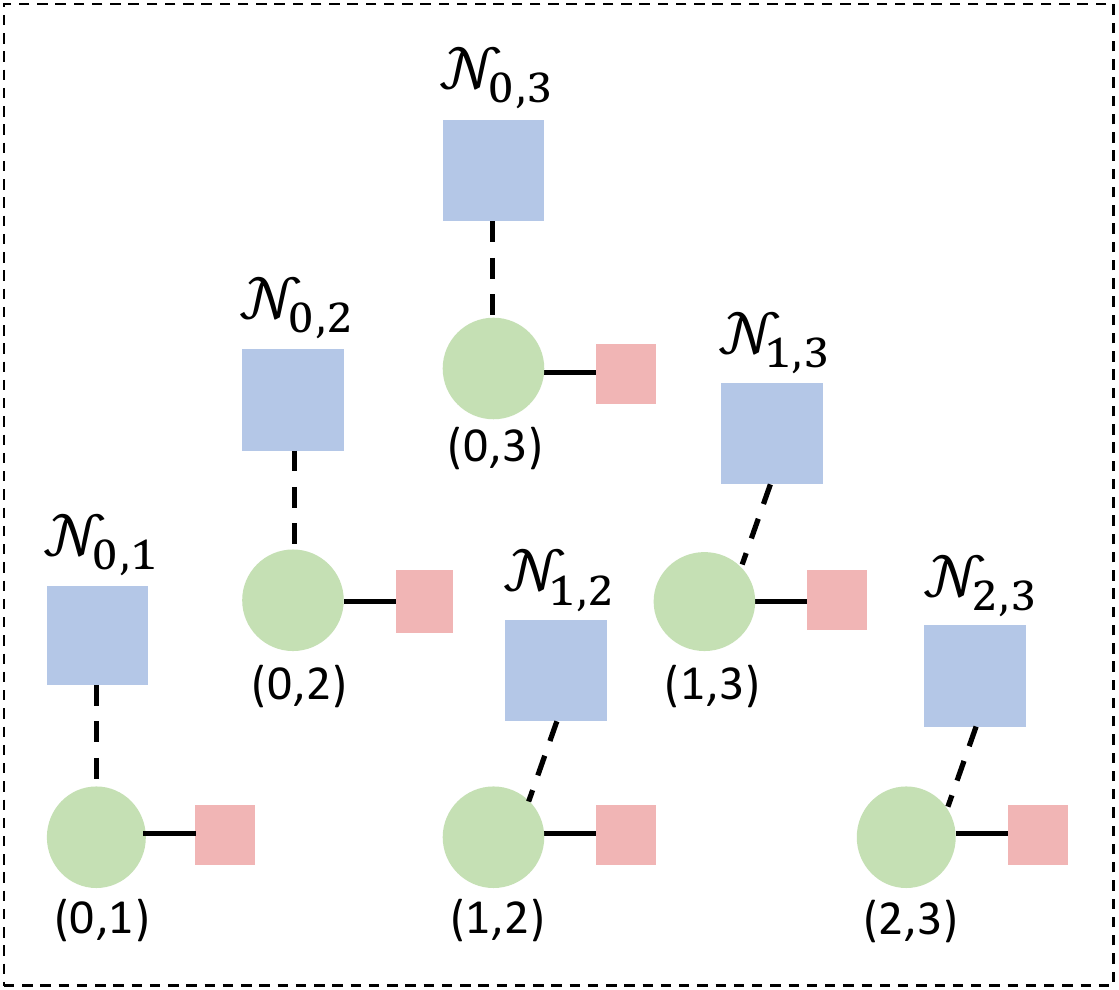}
            \end{minipage}
            \label{Fig:model_2}
        }
        \subfigure[Our normalization.]{
            \begin{minipage}[t]{0.317\textwidth}
                \includegraphics[width=\textwidth]{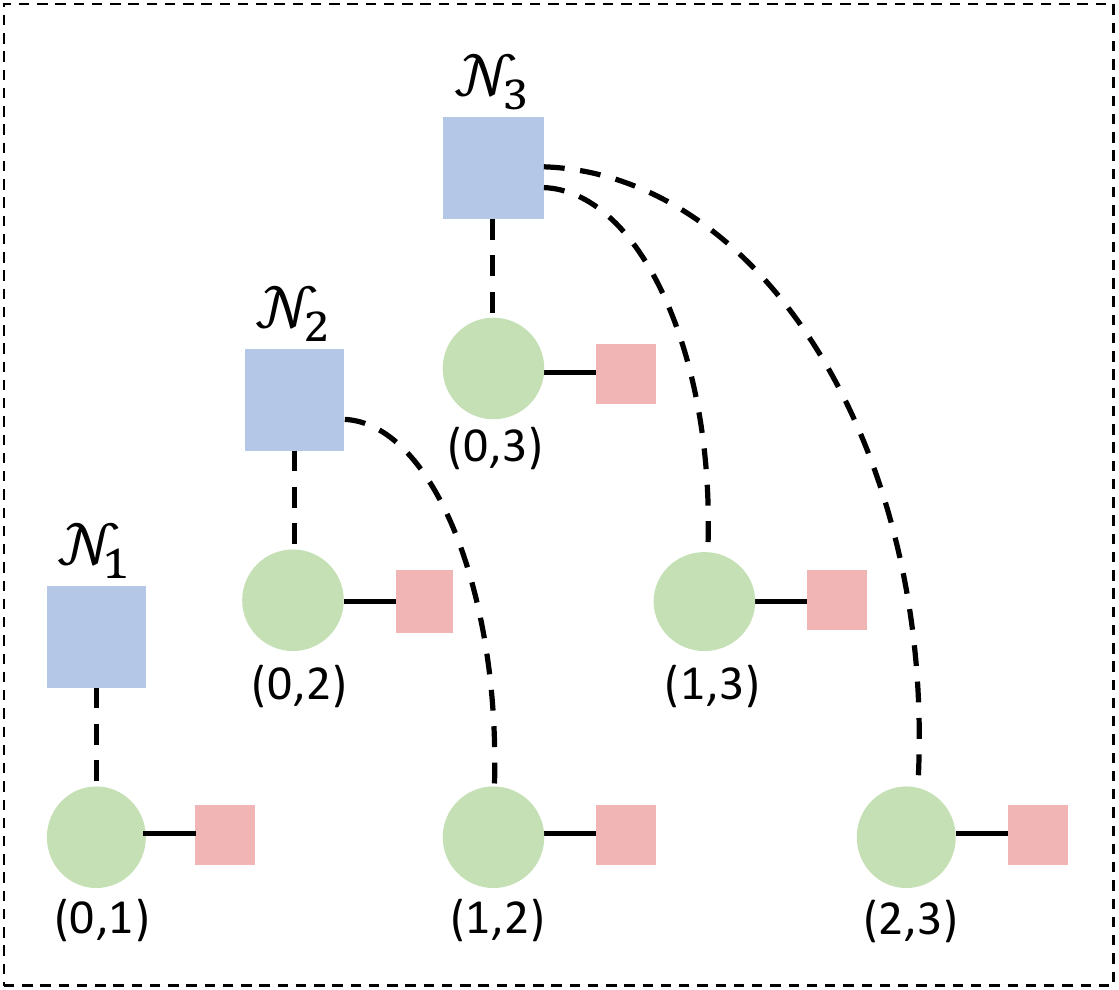}
            \end{minipage}
            \label{Fig:model_3}
        }
        \caption{Factor graphs of three types of normalization.
            Green circles represent all potential spans in the span table.
            Red blocks represent scores of the spans.
            Blue blocks represent normalization operations and dotted lines
            connect all the spans involved in the normalization.
            Global normalization (a) needs to calculate the sum of all span scores
            in parsing tree $\mathcal{T}$.
            Existing local normalization (e.g. \citet{DBLP:conf/coling/TengZ18}) (b)
            only calculates the probability of each candidate span.
            Our method (c) does local normalization on all the spans with the same right boundary. }
        \label{Fig:model}
    \end{figure*}

    \subsection{Decoder}
    \label{Sec:Decoder}
    Since a split point can play two different roles when it is the left
    or right boundary of a span,
    we use two different vectors to represent the two roles
    inspired by \citet{DBLP:conf/iclr/DozatM17}.
    Concretely, we use two multi-layer perceptrons to generate two different representations,
    \begin{equation}
        \label{Eq:splitpointrep}
        \bm{l}_i = \text{MLP}_l(\bm{h}_i),\quad \bm{r}_i = \text{MLP}_r(\bm{h}_i).
    \end{equation}

    Then we can define the score
    of span $(i, j)$ using a biaffine attention function \citep{DBLP:conf/iclr/DozatM17, DBLP:conf/ijcai/LiLZWLS19},
    \[
        \alpha_{ij} = \bm{l}_i^\top {\bf\bm{W}} \bm{r}_j + \bm{b}_1^\top \bm{l}_i + \bm{b}_2^\top \bm{r}_j,
    \]
    where ${\bf\bm{W}}$, $\bm{b}_1$ and $\bm{b}_2$ are all model parameters.
    $\alpha_{ij}$ measures the possibility of $(i, j)$ being a left child span in the tree.

    Different from \citet{DBLP:conf/acl/SternAK17} which does global normalization on the
    probability of the whole tree and \citet{DBLP:conf/coling/TengZ18} which
    does local normalization on each candidate span,
    we do normalization on all spans with the same right boundary $j$.
    Thus the probability of span $(i, j)$ to be a left child span is defined as,
    \begin{equation}
        P(i|j) = \text{Softmax}_{i}(\alpha_{ij}), \forall i < j.
        \label{Eq:normalization}
    \end{equation}
    Finally, we can predict the linearization using the probability $P(i|j)$,
    \begin{equation}
        \label{Eq:pred_linearization}
        d_j = \mathop{\arg\max}_{i} {P(i|j)}, \forall i < j.
    \end{equation}

    For label prediction,
    we first infer the tree structure from the linearization
    (Section \ref{Sec:TreeInference}).
    \footnote{
        Note that we would perform label prediction
        without the tree inference step
        which will train the entire parser in linear time as
        sequence labelling models
        \citep{DBLP:conf/emnlp/Gomez-Rodriguez18},
        but we empirically find that the tree structure helps
        improving the label classifier.
    }
    Then we use a multi-layer perceptron
    to calculate the label probability of span $(i, j)$,
    \[
        P(\ell|i, j) = \text{Softmax}(\text{MLP}_{\text{label}}([\bm{l}_i; \bm{r}_j]))_{\ell}.
    \]
    Final predicted label of span $(i, j)$ is
    $\ell(i, j) = \mathop{\arg\max}_{\ell}{P(\ell|i, j)}$.

    \subsection{Training Objective}
    \label{Sec:TrainingObjective}
    Given a gold parsing tree $\mathcal{T}$ and its linearization $(d_1, d_2, \ldots, d_n)$,
    we can calculate the loss using the negative log-likelihood:
    \[
        \mathcal{L} = -\frac{1}{n} (\sum_{i = 1}^{n}{\log P(d_i | i)} + \sum_{(i, j, \ell) \in \mathcal{T}}{\log P(\ell|i, j)}).
    \]

    The loss function consists of two parts. One is the structure loss, which is only defined on
    the left child spans. The other one is the label loss, which is defined on
    all the spans in $\mathcal{T}$.

    \subsection{Tree Inference}
    \label{Sec:TreeInference}

    To reconstruct the tree structure from the predicted linearization
    $(d_1, d_2, \ldots, d_n)$, we must deal with illegal sequences.
    One solution is to convert an illegal linearization to a
    legal one, and then use Algorithm \ref{Alg:L2T} to recover the tree.
    However, the optimal converting algorithm is NP hard as
    discussed in Section \ref{Sec:Preliminaries}.
    We propose two approximate reconstruction methods,
    both of which are based on replacing line 5 of Algorithm \ref{Alg:L2T}.
    One is to find the largest $k$ satisfying $d_k \le i$,
    \[
        k \gets \max{\{k' \mid d_{k'} \le i, i < k' < j\}}.
    \]
    The other is to find the index $k$ of the smallest $d_{k}$
    (if there are multiple choices, we choose the largest one),
    \[
        k \gets \mathop{\arg\min}_{k'}d_{k'}.
    \]
    Both methods are applicable to legal situations,
    and they have similar performance in our empirical evaluations.
    The inference time complexity is $\mathcal{O}(n^2)$
    in the worst-case for unbalanced trees,
    while in average it is $\mathcal{O}(n \log n)$
    (which is the same as \citet{DBLP:conf/acl/SternAK17}).

    Finally, instead of reconstructing trees from
    linearization sequences $(d_1, d_2, \dots, d_n)$,
    we could have an accurate CKY-style decoding algorithm from
    probabilities $P(i|j)$ (Equation \ref{Eq:normalization}).
    Specifically, it maximizes the product of left child span probabilities,
    \[
        \mathcal{G}(i, j) = \max{\{P(i | k) \times \mathcal{G}(k, j) \mid i < k < j\}},
    \]
    where $\mathcal{G}(i, j)$ represents the highest probability of subtree with root node $(i, j)$.
    We can calculate $\mathcal{G}(0, n)$ using dynamic programming algorithm and
    back-trace the tree accordingly.
    The complexity is $\mathcal{O}(n^3)$.

    \subsection{More Discussions on Normalization}
    We can compare our locally normalized model
    (Equation \ref{Eq:normalization})
    with other probability factorizations of constituent trees
    (Figure \ref{Fig:model}).


    Global normalization (Figure \ref{Fig:model_1})
    performs marginalization over all
    candidate trees, which requires dynamic programming decoding.
    As a local model, our parser is a span-level
    factorization of the tree probability,
    and each factor only marginalizes over
    a linear number of items
    (i.e., the probability of span $(i,j)$ is normalized
    with all scores of $(i', j), i' < j$).
    It is easier to be parallelized and enjoys a much faster parsing speed.
    We will show that its performance is also competitive with global models.

    \citet{DBLP:conf/coling/TengZ18} studies two local normalized models
    over spans, namely the \emph{span model} and the \emph{rule model}.
    The span model simply considers individual spans independently
    (Figure \ref{Fig:model_2}) which may be the finest factorization.
    Our model lies between it and the global model.

    The rule model considers a similar normalization with our model.
    If it is combined with the top-down decoding \citep{DBLP:conf/acl/SternAK17},
    the two parsers look similar. \footnote{
        We thank an anonymous reviewer for pointing out the connection.
        The following discussions are based on his/her detailed reviews.
    }
    We discuss their differences.
    The rule model
    takes all ground truth spans from the gold trees,
    and for each span $(i, j)$, it compiles a probability
    $P((i,j)\gets(i, k)(k, j))$ for its ground truth split $k$.
    Our parser, on the other side, factorizes on each word.
    Therefore, for the same span $(i, j)$, their normalization
    is constrained within $(i, j)$, while ours is over all $i'<j$.
    The main advantage of our parser is simpler span
    representations (not depend on parent spans):
    it makes the parser easy to batch for
    sentences with different lengths and tree structures since each $d_i$
    can be calculated offline before training.

    \section{Experiments}
    \label{Sec:Experiments}

    \begin{table*}[t]
        \normalsize
        \begin{center}
            \begin{tabular}{lccccccccc}
                \hline
                Type       & NP             & VP             & S              & PP             & SBAR           & ADVP           & ADJP           & QP             & WHNP           \\
                Count      & 18630          & 8743           & 5663           & 5492           & 1797           & 1213           & 893            & 490            & 429            \\
                \hline\hline
                PSN Model  & 93.15          & 91.81          & 91.21          & 89.73          & 87.81          & 86.89          & 73.01          & 89.80          & 97.20          \\
                \hline
                Our Model  & \textbf{93.42} & \textbf{92.62} & \textbf{91.95} & \textbf{89.91} & \textbf{88.93} & \textbf{87.39} & \textbf{75.14} & \textbf{91.63} & \textbf{97.44} \\
                Difference & +0.27          & +0.81          & +0.74          & +0.18          & +1.12          & +0.50          & +2.13          & +1.83          & +0.24          \\
                \hline
            \end{tabular}
        \end{center}
        \caption{Comparison on different phrases types.
            Here we only list top nine types.}
        \label{Tab:CType}
    \end{table*}

    \subsection{Data and Settings}
    \paragraph{Datasets and Preprocessing}
    All models are trained on two standard benchmark treebanks,
    English Penn Treebank (PTB) \citep{DBLP:journals/coling/MarcusSM94} and Chinese Penn Treebank (CTB) 5.1.
    The POS tags are predicted using Stanford Tagger \citep{DBLP:conf/naacl/ToutanovaKMS03}.
    To clean the treebanks, we strip the leaf nodes with POS tag \texttt{-NONE-} from the two treebanks
    and delete the root nodes with constituent type \texttt{ROOT}.
    For evaluating the results, we use the standard evaluation tool
    \footnote{\url{http://nlp.cs.nyu.edu/evalb/}}.

    For words in the testing corpus but not in the training corpus,
    we replace them with a unique label \texttt{<UNK>}.
    We also replace the words in the training corpus with
    the unknown label \texttt{<UNK>} with probability
    $p_{\text{unk}}(w) = \frac{z}{z + c(w)}$, where $c(w)$ is the
    number of time word $w$ appears in the training corpus
    and we set $z = 0.8375$ as \citet{DBLP:conf/emnlp/CrossH16}.

    \paragraph{Hyperparameters}
    We use 100D GloVe embedding for PTB \citep{DBLP:conf/emnlp/PenningtonSM14}.
    For character encoding, we randomly initialize the character embeddings
    with dimension 64.

    We use Adam optimizer with initial learning rate 1.0 and epsilon $10^{-9}$.
    For LSTM encoder, we use a hidden size of 1024,
    with 0.2 dropout in all the feed-forward and recurrent connections.
    For Transformer encoder, we use the same hyperparameters as
    \citet{DBLP:conf/acl/KleinK18}.
    For split point representation, we apply two
    1024-dimensional hidden size feed-forward networks.
    All the dropout we use in the decoder layer is 0.33.
    We also use BERT \citep{DBLP:conf/naacl/DevlinCLT19}
    (uncased, 24 layers, 16 attention heads per layer and 1024-dimensional hidden vectors)
    and use the output of the last layer as the pre-trained word embeddings.
    \footnote{The source code for our model is publicly available: \texttt{\url{https://github.com/AntNLP/span-linearization-parser}}}

    \paragraph{Training Details}
    We use \texttt{PyTorch} as our neural network toolkit and
    run the code on a NVIDIA GeForce GTX Titan Xp GPU and Intel Xeon E5-2603 v4 CPU.
    All models are trained for up to 150 epochs with batch size 150 \citep{DBLP:conf/acl/ZhouZ19}.

    \begin{figure}[t]
        \centering
        \includegraphics[width=\linewidth]{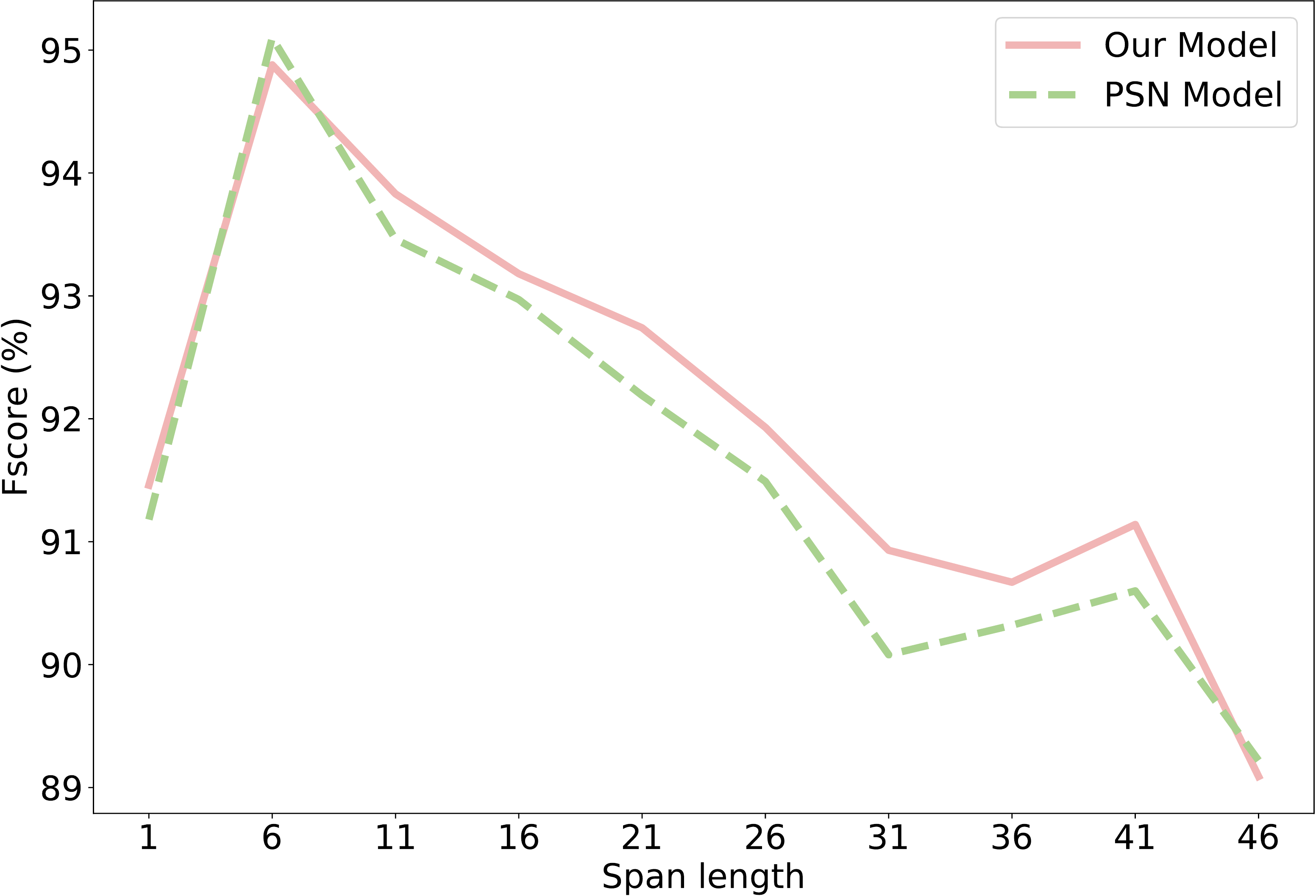}
        \caption{F1 scores against span length.
            Here the length $l$ represents lengths between $[l, l+4]$.}
        \label{Fig:span_length}
    \end{figure}

    \begin{table}[t]
        \normalsize
        \begin{center}
            \resizebox{\linewidth}{!}{
                \begin{tabular}{lccc}
                    \hline
                    Model                                                  & LR            & LP            & F1            \\
                    \hline
                    \textbf{Global Model}                                                                                  \\
                    \hline
                    \citet{DBLP:conf/acl/SternAK17}                        & 90.6          & 93.0          & 91.8          \\
                    \citet{DBLP:conf/naacl/GaddySK18}                      & -             & -             & 92.1          \\
                    \citet{DBLP:conf/acl/KleinK18}$^\spadesuit$            & 93.2          & \textbf{93.9} & 93.6          \\
                    \citet{DBLP:conf/acl/ZhouZ19}$^\spadesuit$             & \textbf{93.6} & \textbf{93.9} & \textbf{93.8} \\
                    \hline
                    \textbf{Local Model}                                                                                   \\
                    \hline
                    \citet{DBLP:conf/naacl/VilaresAS19}                    & -             & -             & 90.6          \\
                    \citet{DBLP:conf/aaai/LiuZS18}                         & -             & -             & 91.2          \\
                    \citet{DBLP:journals/ieicet/MaTLZS17}                  & -             & -             & 91.5          \\
                    \citet{DBLP:conf/acl/BengioSCJLS18}                    & 91.7          & 92.0          & 91.8          \\
                    \citet{DBLP:journals/tacl/LiuZ17a}                     & -             & -             & 91.8          \\
                    \citet{DBLP:conf/acl/HongH18}                          & 91.5          & 92.5          & 92.0          \\
                    \citet{DBLP:conf/coling/TengZ18}                       & 92.2          & 92.5          & 92.4          \\
                    \citet{DBLP:conf/naacl/DyerKBS16}$^\heartsuit$         & -             & -             & 92.4          \\
                    \citet{DBLP:conf/emnlp/SternFK17}$^\heartsuit$         & 92.6          & 92.6          & 92.6          \\
                    \hline
                    \textbf{Our Model}                                     & 92.3          & 92.9          & 92.6          \\
                    \textbf{Our Model}$^\spadesuit$                        & \textbf{93.3} & \textbf{94.1} & \textbf{93.7} \\
                    \hline\hline
                    \multicolumn{4}{l}{\textbf{Pre-training/Ensemble/Re-ranking} }                                         \\
                    \hline
                    \citet{DBLP:conf/aaai/LiuZS18}                         & -             & -             & 92.3          \\
                    \citet{DBLP:conf/emnlp/ChoeC16}                        & -             & -             & 93.8          \\
                    \citet{DBLP:journals/tacl/LiuZ17a}                     & -             & -             & 94.2          \\
                    \citet{DBLP:conf/acl/FriedSK17}                        & -             & -             & 94.7          \\
                    \citet{DBLP:conf/acl/KleinK18}$^\spadesuit$            & 94.9          & 95.4          & 95.1          \\
                    \citet{DBLP:journals/corr/abs-1812-11760}$^\spadesuit$ & 95.5          & 95.7          & 95.6          \\
                    \citet{DBLP:conf/acl/ZhouZ19}$^\spadesuit$             & \textbf{95.7} & 96.0          & \textbf{95.8} \\
                    \hline
                    \textbf{Our Model} (+BERT)                             & 95.6          & 96.0          & \textbf{95.8} \\
                    \textbf{Our Model} (+BERT)$^\spadesuit$                & 95.5          & \textbf{96.1} & \textbf{95.8} \\
                    \hline
                \end{tabular}
            }
        \end{center}
        \caption{ Final results on the PTB test set. $^\spadesuit$ means the models use Transformer as their encoder.
            $^\heartsuit$ means generative models.}
        \label{Tab:PTB_test}
    \end{table}

    \begin{table}[t]
        \normalsize
        \begin{center}
            \resizebox{\linewidth}{!}{
                \begin{tabular}{lccc}
                    \hline
                    Model                                                  & LR            & LP            & F1            \\
                    \hline
                    \textbf{Global Model}                                                                                  \\
                    \hline
                    \citet{DBLP:journals/corr/abs-1812-11760}$^\spadesuit$ & 91.6          & 92.0          & 91.8          \\
                    \citet{DBLP:conf/acl/ZhouZ19}$^\spadesuit$             & 92.0          & 92.3          & 92.2          \\
                    \hline
                    \textbf{Local Model}                                                                                   \\
                    \hline
                    \citet{DBLP:conf/naacl/DyerKBS16}                      & -             & -             & 84.6          \\
                    \citet{DBLP:conf/aaai/LiuZS18}                         & -             & -             & 85.4          \\
                    \citet{DBLP:journals/tacl/LiuZ17}                      & 85.2          & 85.9          & 85.5          \\
                    \citet{DBLP:conf/naacl/VilaresAS19}                    & -             & -             & 85.6          \\
                    \citet{DBLP:journals/tacl/LiuZ17a}                     & -             & -             & 86.1          \\
                    \citet{DBLP:conf/acl/BengioSCJLS18}                    & 86.4          & 86.6          & 86.5          \\
                    \citet{DBLP:conf/acl/FriedK18}                         & -             & -             & 87.0          \\
                    \citet{DBLP:conf/coling/TengZ18}                       & 87.1          & 87.5          & 87.3          \\
                    \hline
                    \textbf{Our Model}                                     & \textbf{92.2} & \textbf{92.7} & \textbf{92.4} \\
                    \textbf{Our Model}$^\spadesuit$                        & 92.1          & 92.3          & 92.2          \\
                    \hline
                \end{tabular}
            }
        \end{center}
        \caption{Final results on the CTB test set. $^\spadesuit$ means the models use Transformer as their encoder.}
        \label{Tab:CTB_test}
    \end{table}

    \subsection{Main Results}
    Table \ref{Tab:PTB_test} shows the final results on PTB test set.
    Our models (92.6 F1 with LSTM, 93.7 F1 with Transformer)
    significantly outperform the single locally normalized models.
    Compared with globally normalized models, our models also outperform those
    parsers with LSTM encoder and achieve a competitive result
    with Transformer encoder parsers.
    With the help of BERT \citep{DBLP:journals/corr/abs-1810-04805},
    our models with two encoders both achieve the same performance (95.8 F1) as
    the best parser \citep{DBLP:conf/acl/ZhouZ19}.
    Table \ref{Tab:CTB_test} shows the final results on CTB test set.
    Our models (92.4 F1) also significantly outperform local models and
    achieve competitive result amongst global models.

    Compared with \citet{DBLP:conf/coling/TengZ18} which does local normalization on single span,
    our model increases 0.2 F1 on PTB,
    which shows that doing normalization on more spans is really better.
    Our model also significantly outperforms \citet{DBLP:conf/acl/BengioSCJLS18}
    which predicts the syntactic distance of a tree.
    This indicates the superiority of our linearization method directly tied on the spans.

    \subsection{Evaluation}
    \label{Evaluation}
    To better understand the extent to which our model transcends
    the locally normalized model which does normalization on a single
    span described in \citet{DBLP:conf/coling/TengZ18},
    we do several experiments to compare the performance about
    different lengths of spans and different constituent types.

    In order to make a fair comparison, we implement their model
    by ourselves using the same LSTM encoder as ours.
    Besides, we ignore the LSTM for label prediction and complex span representations
    in their models and use simpler settings.
    Our own implementation achieves the same result as they report (92.4 F1).
    For convenience, we call their model per-span-normalization
    (PSN for short) model in the following.

    \paragraph{Influence of Span Length}
    First, we analyse the influence of different lengths of spans
    and the results are shown in Figure \ref{Fig:span_length}.
    We find that for sentences of lengths between $[11, 45]$,
    our model significantly outperforms PSN model.
    For short spans, PSN model only needs to consider few spans,
    which is more local and it is enough for the per-span-normalization
    to handle this situation.
    For long spans, our model needs to do normalization on more spans
    and the state space becomes large linearly.
    So the accuracy decreases fast, and there is no advantage compared with
    PSN model which uses CKY algorithm for inference.
    For spans of other lengths, our locally normalized method
    can take all spans with the same right boundary into consideration
    and add sum-to-one constraints on their scores.
    As a result, our model outperforms PSN
    model even without the help of accurate inference.

    \begin{table}[t]
        \normalsize
        \begin{center}
            \begin{tabular}{lccc}
                \hline
                Model                               & LR             & LP             & F1             \\
                \hline
                Full model                          & \textbf{92.31} & 92.87          & \textbf{92.59} \\
                \hline\hline
                - $\text{MLP}_l$ and $\text{MLP}_r$ & 92.15          & 92.72          & 92.43          \\
                - normalization                     & 91.25          & \textbf{92.93} & 92.08          \\
                + label linearization               & 90.79          & 91.56          & 91.17          \\
                \hline
            \end{tabular}
        \end{center}
        \caption{Ablation test on the PTB test set.
            Here we use the same settings as in Section \ref{Evaluation}.}
        \label{Tab:Ablation}
    \end{table}

    \paragraph{Influence of Constituent Type}
    Then we compare the accuracy of different constituent types.
    Table \ref{Tab:CType} shows the results of nine types which occur
    most frequently.
    Our model all performs better than PSN model,
    especially in types SBAR, ADJP and QP.
    When optimizing the representation of one split point,
    our model can consider all of the words before it,
    which can be helpful to predict some types.
    For example, when we predict an adjective phrase (ADJP),
    its representation has fused the words' information before it
    (e.g. linking verb like ``\emph{is}''), which can narrow the scope of prediction.

    \subsection{Ablation Study}

    \begin{table}[t]
        \normalsize
        \begin{center}
            \resizebox{\linewidth}{!}{
                \begin{tabular}{lccc}
                    \hline
                    Inference Algorithm                   & LR             & LP             & F1             \\
                    \hline
                    $\mathcal{G}(i, j)$                   & 92.31          & 92.87          & \textbf{92.59} \\
                    $k = \max{\{k' \mid d_{k'} \le i\}}$  & \textbf{92.39} & 92.75          & 92.57          \\
                    $k = \mathop{\arg\min}_{k'} {d_{k'}}$ & 91.93          & \textbf{93.21} & 92.57          \\
                    \hline
                \end{tabular}
            }
        \end{center}
        \caption{Results of different inference algorithms described in Section \ref{Sec:TreeInference}.}
        \label{Tab:Inference}
    \end{table}

    \begin{table}[t]
        \normalsize
        \begin{center}
            \resizebox{\linewidth}{!}{
                \begin{tabular}{lc}
                    \hline
                    Model                                                            & sents/sec \\
                    \hline
                    \textbf{Global Model}                                                        \\
                    \hline
                    \citet{DBLP:conf/acl/SternAK17}                                  & 20        \\
                    \citet{DBLP:conf/acl/KleinK18}$^\spadesuit$ (w. \texttt{Cython}) & 150       \\
                    \citet{DBLP:conf/acl/ZhouZ19}$^\spadesuit$ (w. \texttt{Cython})  & 159       \\
                    \hline
                    \textbf{Local Model}                                                         \\
                    \hline
                    \citet{DBLP:conf/coling/TengZ18}                                 & 22        \\
                    \citet{DBLP:conf/acl/SternAK17}                                  & 76        \\
                    \citet{DBLP:journals/tacl/LiuZ17}                                & 79        \\
                    \citet{DBLP:conf/acl/BengioSCJLS18}                              & 111       \\
                    \citet{DBLP:conf/acl/BengioSCJLS18} (w/o tree inference)         & 351       \\
                    \citet{DBLP:conf/naacl/VilaresAS19}                              & 942       \\
                    \hline
                    \textbf{Our Model}                                               & 220       \\
                    \textbf{Our Model}$^\spadesuit$                                  & 155       \\
                    \hline
                \end{tabular}
            }
        \end{center}
        \caption{Parsing speeds on the PTB test set. $^\spadesuit$ means the models use Transformer as their encoders.
            ``w. \texttt{Cython}'' stands for using \texttt{Cython} to optimize the python code.
            ``w/o tree inference'' stands for evaluating without tree inference.
            The model in \citet{DBLP:conf/acl/KleinK18} is ran by ourselves,
            and other speeds are extracted from their original papers.}
        \label{Tab:speed}
    \end{table}

    We perform several ablation experiments by modifying the structure of
    the decoder layer.
    The results are shown in Table \ref{Tab:Ablation}.

    First, we delete the two different split point representations described
    in Equation \eqref{Eq:splitpointrep} and directly use the output of LSTM
    as the final representation.
    Final performance slightly decreases, which indicates that distinguishing the
    representations of left and right boundaries of a span is really helpful.

    Then we delete the local normalization on partial spans
    and only calculate the probability of each span to be a left child.
    The inference algorithm is the same as our full model.
    Final result decreases by 0.5 F1, despite improvement on precision.
    This might be because our normalization method can add constraints on all the spans
    with the same right boundary,
    which makes it effective when only one span is correct.

    Finally, we try to predict the labels sequentially,
    which means assigning each split $i$ a tuple $(d_i, \ell_i^{\text{left}}, \ell_i^{\text{right}})$,
    where $\ell_i^{\text{left}}$ and $\ell_i^{\text{right}}$ represent the labels
    of the longest spans ending and starting with $i$ in the tree, respectively.
    This may make our model become a sequence labeling model
    similar to \citet{DBLP:conf/emnlp/Gomez-Rodriguez18}.
    However, the performance is very poor, and this is largely
    due to the loss of structural information in the label prediction.
    Therefore, how to balance efficiency and label prediction accuracy
    might be a research problem in the future.

    \subsection{Inference Algorithms}
    We compare three inference algorithms described in Section \ref{Sec:TreeInference}.
    The results are shown in Table \ref{Tab:Inference}.
    We find that different inference algorithms
    have no obvious effect on the performance,
    mainly due to the powerful learning ability of our model.
    Thus we use the third method which is the most convenient to implement.

    \subsection{Parsing Speed}
    The parsing speeds of our parser and other parsers are shown in Table \ref{Tab:speed}.
    Although our inference complexity is $\mathcal{O}(n \log n)$,
    our speed is faster than other local models,
    except \citet{DBLP:conf/acl/BengioSCJLS18} which evaluates without tree inference
    and \citet{DBLP:conf/naacl/VilaresAS19} which utilizes a pure sequence tagging framework.
    This is mainly due to the simplicity of our model
    and the parallelism of matrix operations for structure prediction.
    Compared with globally normalized parsers like \citet{DBLP:conf/acl/ZhouZ19} and \citet{DBLP:conf/acl/KleinK18},
    our model is also faster even if they use optimization for python code (e.g. \texttt{Cython}
    \footnote{\texttt{\url{https://cython.org/}}}).
    Other global model like \citet{DBLP:conf/acl/SternAK17}
    which infers in $O(n^3)$ complexity is much slower than ours,
    and this shows the superiority of our linearization in speed.

    \section{Related Work}
    \label{Sec:RelatedWork}
    Globally normalized parsers often have high performance on constituent parsing
    due to their search on the global state space
    \citep{DBLP:conf/acl/SternAK17, DBLP:conf/acl/KleinK18, DBLP:conf/acl/ZhouZ19}.
    However, they suffer from high time complexity and are difficult to parallelize.
    Thus many efforts have been made to optimize their efficiency \citep{DBLP:journals/tacl/VieiraE17}.

    Recently, the rapid development of encoders
    \citep{DBLP:journals/neco/HochreiterS97, DBLP:conf/nips/VaswaniSPUJGKP17}
    and pre-trained language models
    \citep{DBLP:journals/corr/abs-1810-04805}
    have enabled local models to achieve similar performance as global models.
    \citet{DBLP:conf/coling/TengZ18} propose two local models,
    one does normalization on each candidate span and one on each grammar rule.
    Their models even outperform the global model in \citet{DBLP:conf/acl/SternAK17}
    thanks to the better representation of spans.
    However, they still need an $\mathcal{O}(n^3)$ complexity inference algorithm
    to reconstruct the final parsing tree.

    Meanwhile, many work do research on faster sequential models.
    Transition-based models predict a sequence of actions and achieve
    an $\mathcal{O}(n)$ complexity
    \citep{DBLP:conf/acl/WatanabeS15, DBLP:conf/emnlp/CrossH16, DBLP:journals/tacl/LiuZ17a}.
    However, they suffer from the issue of error propagation and cannot
    be parallel.
    Sequence labeling models regard tree prediction as sequence prediction problem
    \citep{DBLP:conf/emnlp/Gomez-Rodriguez18, DBLP:conf/acl/BengioSCJLS18}.
    These models have high efficiency, but their linearizations have no direct
    relation to the spans, so the performance is much worse than span-based models.

    We propose a novel linearization method closely related to the spans
    and decode the tree in $\mathcal{O}(n \log n)$ complexity.
    Compared with \citet{DBLP:conf/coling/TengZ18}, we do normalization on more spans,
    thus achieve a better performance.

    In future work, we will apply graph neural network
    \citep{DBLP:conf/iclr/VelickovicCCRLB18, DBLP:conf/acl/JiWL19, DBLP:conf/acl/SunGWGJLSD19}
    to enhance the span representation.
    Due to the excellent properties of our linearization,
    we can jointly learn constituent parsing and dependency parsing
    in one graph-based model.
    In addition, there is also a right linearization
    defined on the set of right child spans.
    We can study how to combine the two linear representations to
    further improve the performance of the model.

    \section{Conclusion}
    In this work, we propose a novel linearization of constituent trees
    tied on the spans tightly.
    In addition, we build a new normalization method,
    which can add constraints on all the spans with the same right boundary.
    Compared with previous local normalization methods,
    our method is more accurate for considering more span information,
    and reserves the fast running speed due to the parallelizable linearization model.
    The experiments show that our model significantly outperforms existing local models
    and achieves competitive results with global models.

    \section*{Acknowledgments}

    The authors would like to thank the reviewers for their helpful comments and suggestions.
    The authors would also like to thank Tao Ji and Changzhi Sun
    for their advices on models and experiments.
    The corresponding author is Yuanbin Wu.
    This research is
    (partially) supported by STCSM (18ZR1411500),
    the Foundation of State Key Laboratory of Cognitive Intelligence,
    iFLYTEK(COGOS-20190003),
    and an open research fund of KLATASDS-MOE.

    \bibliography{acl2020}
    \bibliographystyle{acl_natbib}

\end{CJK*}
\end{document}